\documentclass{article}
\usepackage{ijcai22}

\usepackage{times}
\usepackage{soul}
\usepackage{url}
\usepackage[hidelinks]{hyperref}
\usepackage[utf8]{inputenc}
\usepackage[small]{caption}
\usepackage{graphicx}
\usepackage{amsmath}
\usepackage{amsthm}
\usepackage{booktabs}
\usepackage{algorithm}
\usepackage{algorithmic}
\usepackage{bm}
\usepackage[whole]{bxcjkjatype}
\usepackage{amssymb}
\usepackage{stmaryrd}
\hypersetup{draft}

\newtheorem{example}{Example}
\newtheorem{theorem}{Theorem}
\newtheorem{definition}{Definition}

\newtheorem{proposition}{Proposition}

\DeclareMathOperator*{\argmax}{arg\,max}

\title{Towards Unifying Logical Entailment and Statistical Estimation}

\author{Hiroyuki Kido
\affiliations
Cardiff University
\emails
KidoH@cardiff.ac.uk
}

\begin{document}
\maketitle

%%%%%%%%%%%%%%%%%%%%%%%%%%%%%%%%%%%%%%%%%%%%%%%%%%%%%%%%%%%%%%%%%%%%%%%%%%%%%%%%%%%%%%%%%%%%%%%%%%%%%%%%%%%%%%%%%%%%%%%%%%%%%%%%%%%%%%%%%%%%%%%%%%%%%%%%%%%%%%%%%%%%%%%%%%%%%%%%%%%%%%%%%%%%%%%%%%%%%%%%%%%%%%%%%%%%%%%%%%%%%%%%%%%%%%%%%%%%%%%%%%%%%%%
\begin{abstract}
This paper gives a generative model of the interpretation of formal logic for data-driven logical reasoning. The key idea is to represent the interpretation as likelihood of a formula being true given a model of formal logic. Using the likelihood, Bayes' theorem gives the posterior of the model being the case given the formula. The posterior represents an inverse interpretation of formal logic that seeks models making the formula true. The likelihood and posterior cause Bayesian learning that gives the probability of the conclusion being true in the models where all the premises are true. This paper looks at statistical and logical properties of the Bayesian learning. It is shown that the generative model is a unified theory of several different types of reasoning in logic and statistics.
\end{abstract}

%%%%%%%%%%%%%%%%%%%%%%%%%%%%%%%%%%%%%%%%%%%%%%%%%%%%%%%%%%%%%%%%%%%%%%%%%%%%%%%%%%%%%%%%%%%%%%%%%%%%%%%%%%%%%%%%%%%%%%%%%%%%%%%%%%%%%%%%%%%%%%%%%%%%%%%%%%%%%%%%%%%%%%%%%%%%%%%%%%%%%%%%%%%%%%%%%%%%%%%%%%%%%%%%%%%%%%%%%%%%%%%%%%%%%%%%%%%%%%%%%%%%%%%
\section{Introduction}
Thanks to big data and computational power available today, Bayesian statistics plays an important role in various fields such as neuroscience, cognitive science and artificial intelligence (AI). Bayesian brain hypothesis \cite{knill:04}, free-energy principle \cite{friston:10} and predictive coding \cite{hohwy:08} argue that probabilistic reasoning using Bayes' theorem or its approximation explains some higher-order cognitive functions of the cerebral cortex such as perception, action and learning. The common idea is that the brain is a generative model that actively predicts and perceives the world using the belief of states of the world. Bayes' theorem here defines how sensory inputs such as sight, sound, smell, taste and touch update the belief.
\par
% Formal logic provides an entailment, which defines what conclusions follow logically from premises.
Formal logic concerns the laws of human rational thought. The Bayesian brain hypothesis would therefore result in another hypothesis that there is a Bayesian algorithm and data-structure for logical reasoning. This hypothesis is important for the following reasons. First, it has a potential to cause a mathematical model to explain how the human brain performs logical reasoning. Second, the existence of such a model supports the Bayesian brain hypothesis in terms of formal logic. Third, such a model gives a way to critically assess the existing formalisms of logical reasoning. Nevertheless, few research has focused on reformulating logical reasoning in terms of Bayesian perspectives. Bayesian networks \cite{pearl:88}, probabilistic relational models (PRM) \cite{friedman:99}, probabilistic logic programming (PLP) \cite{sato:95} and Markov logic networks (MLN) \cite{richardson:06} are a few exceptions intrinsically relating to Bayesian inference. However, none of them aims to model the process by which data about states of the world generate models of formal logic and then the models generate the truth values of logical formulae. Such a model should give a unified way to deal with the tasks of the statistics and logic shown in Figure \ref{fig:LandS}. This is an important problem because various challenging AI problems such as grounding, frame problems, knowledge acquisition bottleneck, commonsense reasoning and contextual adaptation, come from their disconnection.
%
%Bayesian account of logical reasoning is important. First, it has a potential to be a mathematical model to explain how the human brain performs logical reasoning. Second, it supports the Bayesian brain hypothesis in an analytical way in terms of logic. Third, it gives a way to critically assess the existing formalisms of logical reasoning.
%
% Here, logic concerns a logical consequence relation between sentences whereas statistics concerns whether sentences reflect aspects of the real world.
\par
Formal logic considers an interpretation on each model (denoted by $m$), which represents a state of the world. The interpretation is a function that maps each formula (denoted by $\alpha$) to a truth value, which represents knowledge of the world. Our idea is to give a generative model of the interpretation and use it to probabilistically generate knowledge from data about states of the world. The most basic theoretical idea is to represent the interpretation as likelihood $p(\alpha|m)$. Using the likelihood, Bayes' theorem gives posterior $p(m|\alpha)$, which represents an inverse interpretation that gives the probability that the model making formula $\alpha$ true is $m$. The likelihood and posterior cause Bayesian learning $p(\alpha|\beta)=\sum_{m}p(\alpha|m)p(m|\beta)$, which gives the probability of the formula $\alpha$ being true in the models where the formula $\beta$ is true. This paper studies statistical and logical properties of the Bayesian learning. 
\par
This paper is organised as follows. Section 2 introduces a generative model for logical consequence relations. Section 3 shows logical and statistical correctness of the generative model. Section 4 concludes with a discussion of limitations and future work.
%\par
%This paper has the following three contributions.
%
%First, it proposes the Bayesian entailment hypothesis as a consequence of the Bayesian brain hypothesis. It contributes to launch a new interdisciplinary project across neuroscience and formal logic.
%
%Second, it introduces a generative model unifying classical reasoning, paraconsistent reasoning, counterfactual reasoning and nonmonotonic reasoning. This contributes not only to justifying the Bayesian entailment hypothesis but also to advocating the Bayesian brain hypothesis. 
%
%Third, this paper empirically shows the reasonable learning performance of the generative model in classification tasks. This contributes to provide an approach to unify reasoning and learning, which remains an important open question in the field.
%
%\par
%summarises this paper.
%This paper is organised as follows. Section 2 introduces a generative model for logical consequence relations. Section 3 shows the logical and machine learning correctness of the generative model. Section 4 concludes with a discussion of related work.
%%%%%%%%%
\begin{figure}[t]
\begin{center}
 \includegraphics[scale=0.33]{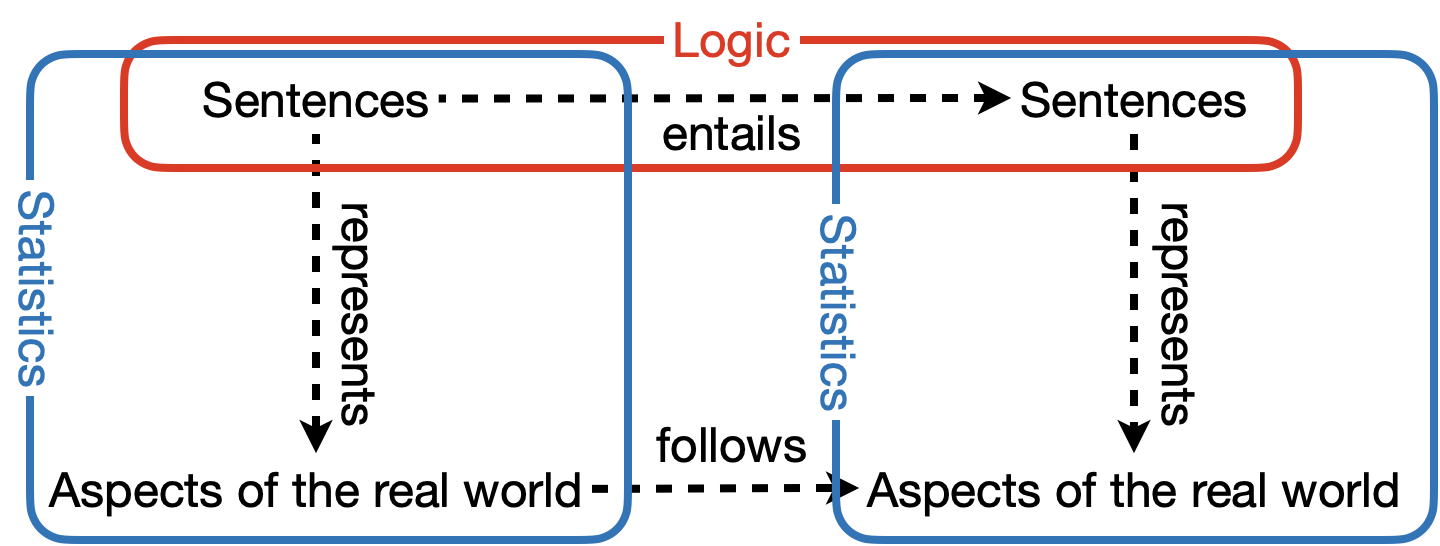}
 \caption{Logic concerns if sentences follow logically from other whereas statistics concerns if they reflect aspects of the real world.}
 \label{fig:LandS}
\end{center}
\end{figure}
%%%%%%%%%
%%%%%%%%%%%%%%%%%%%%%%%%%%%%%%%%%%%%%%%%%%%%%%%%%%%%%%%%%%%%%%%%%%%%%%%%%%%%%%%%%%%%%%%%%%%%%%%%%%%%%%%%%%%%%%%%%%%%%%%%%%%%%%%%%%%%%%%%%%%%%%%%%%%%%%%%%%%%%%%%%%
%\section{Methods}
\section{Method}
%Let $L$ be a logical language constructed with a finite set of atomic formulae. The language is logical in the sense that it is defined as usual using typical logical connectives such as $\lnot$, $\land$, $\lor$, $\rightarrow$ and $\leftrightarrow$. We assume no further syntax, e.g., propositional or first-order language, in order to keep our formalism as general as possible.
%
%We assume that ${\cal D}$ determines the possible states of the world and that each state of the world in ${\cal D}$ occurs with the same probability. 
%
Let ${\cal D}=\{d_{1},d_{2},...,d_{K}\}$ be a multiset of data about states of the world. $D$ is a random variable whose realisations are data in ${\cal D}$. For all data $d_{k}\in{\cal D}$, we define the probability of $d_{k}$, as follows.
\begin{eqnarray*}
p(D=d_{k})=\frac{1}{K}
\end{eqnarray*}
%and $p(D)$ is a uniform categorical distribution, i.e., $p(D=d_{k})=1/K$, for all $d_{k}\in{\cal D}$.
%
%Let $L$ be a propositional or first-order language. For the sake of simplicity, we assume that $L$ has no function symbol or open formula. ${\cal D}=\{d_{1},d_{2},...,d_{K}\}$ represents a set of data about states of the world. We assume that ${\cal D}$ determines possible states of the world and that each state of the world in ${\cal D}$ occurs with the same probability. $D$ is a random variable whose realisations are data in ${\cal D}$. $p(D)$ is a uniform categorical distribution, i.e., $p(D=d_{k})=1/K$, for all $d_{k}\in{\cal D}$.
%
$L$ represents a propositional or first-order language. For the sake of simplicity, we assume no function symbol or open formula in $L$. ${\cal M}=\{m_{1},m_{2},...,m_{N}\}$ is a set of models in formal logic. ${\cal D}$ is assumed to be complete with respect to ${\cal M}$, and thus each data in ${\cal D}$ belongs to a single model in ${\cal M}$. $m$ is a function that maps each data to such a single model. $K_{n}$ denotes the number of data that belongs to $m_{n}$, i.e., $K_{n}=|\{d_{k}\in{\cal D}|m_{n}=m(d_{k})\}|$ where $|X|$ for set $X$ denotes the cardinality of $X$. $M$ is a random variable whose realisations are models in ${\cal M}$. For all models $m_{n}\in{\cal M}$ and data $d_{k}\in{\cal D}$, we define the conditional probability of $m_{n}$ given $d_{k}$, as follows.
\begin{eqnarray*}
&&p(M=m_{n}|D=d_{k})=
\begin{cases}
1 & \text{if } m_{n}=m(d_{k})\\
0 & \text{otherwise }
\end{cases}
\end{eqnarray*}
Formal logic considers an interpretation on each model. The interpretation is a function that maps each formula to a truth value, which represents knowledge of the world. We here introduce parameter $\mu\in[0,1]$ to represent the extent to which each model is taken for granted in the interpretation. Concretely, $\mu$ denotes the probability that a formula is interpreted as being true (resp. false) in a model where it is true (resp. false). $1-\mu$ is therefore the probability that a formula is interpreted as being true (resp. false) in a model where it is false (resp. true). We assume that each formula is a random variable whose realisations are 0 and 1, denoting false and true, respectively. For all models $m_{n}\in{\cal M}$ and formulae $\alpha\in L$, we define the conditional probability of each truth value of $\alpha$ given $m_{n}$, as follows.
\begin{eqnarray*}
&&p(\alpha=1|M=m_{n})=
\begin{cases}
\mu & \text{if } m_{n}\in\llbracket\alpha=1\rrbracket\\
1-\mu & \text{otherwise }
\end{cases}
\\
&&p(\alpha=0|M=m_{n})=
\begin{cases}
\mu & \text{if } m_{n}\in\llbracket\alpha=0\rrbracket\\
1-\mu & \text{otherwise }
\end{cases}
\end{eqnarray*}
Here, $\llbracket\alpha=1\rrbracket$ denotes the set of all models in which $\alpha$ is true, and $\llbracket\alpha=0\rrbracket$ the set of all models in which $\alpha$ is false. The above expressions can be simply written as a Bernoulli distribution with parameter $\mu\in[0,1]$, i.e.,
\begin{eqnarray*}
p(\alpha|M=m_{n})=\mu^{\llbracket\alpha\rrbracket_{m_{n}}}(1-\mu)^{1-\llbracket\alpha\rrbracket_{m_{n}}}.
\end{eqnarray*}
Here, $\llbracket\alpha\rrbracket_{m_{n}}$ is a function such that $\llbracket\alpha\rrbracket_{m_{n}}=1$ if $m_{n}\in\llbracket\alpha\rrbracket$ and $\llbracket\alpha\rrbracket_{m_{n}}=0$ otherwise. Recall that $\alpha$ is a random variable, and thus $\llbracket\alpha\rrbracket_{m_{n}}$ is either $\llbracket\alpha=0\rrbracket_{m_{n}}$ or $\llbracket\alpha=1\rrbracket_{m_{n}}$.
\par
In classical logic, given a model, the truth value of each formula is independently determined. In probability theory, this means that the truth values of any two formulae $\alpha_{1}$ and $\alpha_{2}$ are conditionally independent given a model $m_{n}$, i.e., $p(\alpha_{1},\alpha_{2}|M=m_{n})=p(\alpha_{1}|M=m_{n})p(\alpha_{2}|M=m_{n})$. Note that the conditional independence holds not only for atomic formulae but for compound formulae as well.\footnote{In contrast, independence, i.e., $p(\alpha_{1},\alpha_{2})= p(\alpha_{1})p(\alpha_{2})$, holds only for atomic formulae.} Let $\Delta=\{\alpha_{1},\alpha_{2},...,\alpha_{J}\}$ be a multiset of $J$ formulae. We thus have
\begin{eqnarray*}
p(\Delta|M=m_{n})=\prod_{j=1}^{J}p(\alpha_{j}|M=m_{n}).
\end{eqnarray*}
\par
Thus far, we have defined $p(D)$ and $p(M|D)$ as categorical distributions and $p(\Delta|M)$ as Bernoulli distributions with parameter $\mu$. Given a value of the parameter $\mu$, they provide the full joint distribution over all of the random variables, i.e. $p(\Delta,M,D)$. We call $\{p(\Delta|M,\mu), p(M|D), p(D)\}$ a logical model. In sum, the logical model defines a data-driven interpretation by which the truth values of formulae are logically interpreted and probabilistically generated from models. The models are also probabilistically generated from data observed from the real world. The logical model meets the following important properties.
\begin{proposition}\label{kolmogorov}
%The probability of a truth value of a formula satisfies Kolmogorov's axioms.
The logical model satisfies Kolmogorov's axioms.
\end{proposition}
%\begin{proof}
%See Appendices.
%\end{proof}
%
%The next proposition shows that the logical model is sound in terms of logical negation.
%
\begin{proposition}\label{negation}
Let $\alpha\in L$. $p(\alpha=0)=p(\neg\alpha=1)$ holds.
\end{proposition}
%\begin{proof}
%See Appendices.
%\end{proof}
%
In the following, we therefore replace $\alpha=0$ by $\lnot\alpha=1$ and then abbreviate $\lnot\alpha=1$ to $\lnot\alpha$. We also abbreviate $M=m_{n}$ to $m_{n}$ and $D=d_{k}$ to $d_{k}$. 
%Let us see an example in propositional logic.
%
\begin{table}[t]
\begin{minipage}[c]{.45\hsize}
\centering
\setlength{\tabcolsep}{0.5mm} 
\caption{Models and data.}
\label{ex:data}
\begin{tabular}{c|cc|c}
& $rain$ & $wet$ & data ${\cal D}$\\\hline
$m_{1}$ & $0$ & $0$ & $\times\times\times\times$\\
$m_{2}$ & $0$ & $1$ & $\times\times$\\
$m_{3}$ & $1$ & $0$ & $\times$\\
$m_{4}$ & $1$ & $1$ & $\times\times\times$
\end{tabular}
\end{minipage}
\begin{minipage}[c]{.45\hsize}
\centering
\setlength{\tabcolsep}{0.5mm} 
\caption{Likelihoods.}
\label{ex:likelihoods}
\begin{tabular}{c|cc}
& $p(rain|M)$ & $p(wet|M)$\\\hline
$m_{1}$ & $1-\mu$ & $1-\mu$\\
$m_{2}$ & $1-\mu$ & $\mu$\\
$m_{3}$ & $\mu$ & $1-\mu$\\
$m_{4}$ & $\mu$ & $\mu$
\end{tabular}
\end{minipage}
\end{table}
%
%\begin{figure}[t]
%\begin{center}
% \includegraphics[scale=0.37]{figs/BN6.png}
% \caption{Instance of the logical model.}
%  \label{dependency2}
%\end{center}
%\end{figure}
%
%%%%%%%%%%%%%%%%%%%%%%%%%%%%%%%%%%%%%%%%%%%%%%%%%%%%%%%%%%%%%%%%%%%%%%%%%%%%%%%%%%
\begin{example}\label{ex:BE}%[Bayesian entailment]
Let $rain$ and $wet$ be two propositional symbols meaning `it is raining' and `the grass is wet,' respectively. Each row of Table \ref{ex:data} shows a different model, i.e., valuation. The last column shows how many data belongs to each model. Table \ref{ex:likelihoods} shows the likelihoods of the atomic propositions being true given a model. Given $\{p(\Delta|M,\mu=1), p(M|D), p(D)\}$, we have
%the probability of $rain$ given $wet$ is calculated as follows.
%
\begin{eqnarray*}
%&=&\frac{p(wet, rain\rightarrow wet)}{p(rain\rightarrow wet)}
&&p(rain|wet)\\
&&=\frac{\sum_{n=1}^{N}p(rain|m_{n})p(wet|m_{n})\sum_{k=1}^{K}p(m_{n}|d_{k})p(d_{k})}{\sum_{n=1}^{N}p(wet|m_{n})\sum_{k=1}^{K}p(m_{n}|d_{k})p(d_{k})}\\
&&=\frac{\sum_{n=1}^{N}p(rain|m_{n})p(wet|m_{n})\frac{K_{n}}{K}}{\sum_{n=1}^{N}p(wet|m_{n})\frac{K_{n}}{K}}\\
&&=\frac{(1-\mu)^{2}\frac{4}{10}+(1-\mu)\mu\frac{2}{10}+\mu(1-\mu)\frac{1}{10}+\mu^{2}\frac{3}{10}}{(1-\mu)\frac{4}{10}+\mu\frac{2}{10}+(1-\mu)\frac{1}{10}+\mu\frac{3}{10}}\\
&&=\frac{3}{2+3}=0.6.
%
%&=&\frac{0.3\mu^{2}+(0.2+0.1)\mu(1-\mu)+0.4(1-\mu)^{2}}{(0.2+0.3)\mu+(0.4+0.1)(1-\mu)}\\
%&=&\frac{0.4\mu^{2}-0.5\mu+0.4}{0.5}=0.6
\end{eqnarray*}
%
%Therefore, $\{wet\}\vapprox_{\theta} rain$ holds iff $\theta\leq 0.6$. Note that the same result is obtained using the right table in Figure \ref{fig:data}.\footnote{In this case, the prior probability has no effect on the calculation, as it disappears in the fraction reduction. Specifically, the presence of data is only needed to calculate a predictive probability.} Figure \ref{dependency2} shows the dependency over the random variables (denoted by circles) and the parameters (denoted by dots) used in this calculation.
\end{example}
%
%%%%%%%%%%%%%%%%%%%%%%%%%%%%%%%%%%%%%%%%%%%%%%%%%%%%%%%%%%%%%%%%%%%%%%%%%%%%%%%%%%%%%%%%%%%%%%%%%%%%%%%%%%%%%%%%%%%%%%%%%%%%%%%%%%%%%%%%%%%%%%%%%%%%%%%%%%%%%%%%%%%%%%%%%%%%%%%%%%%%%%%%%%%%%%%%%%%%%%%%%%%%%%%%%%%%%%%%
\begin{example}
Suppose that $L$ has only one 2-ary predicate symbol `$blames$' and that the Herbrand universe for $L$ has only two constants $\{a,b\}$. There are four ground atoms, $\{blames(a,a), blames(a,b)$, $blames(b,a)$, $blames(b,b)\}$, which result in $2^{4}=16$ possible models. Each row of Table \ref{tab:FOL} shows a different model and the last column shows the number of data that belongs to the model. Models without data are abbreviated from the table. Given $\{p(\Delta|M,\mu=1), p(M|D), p(D)\}$, we have
%
%今，エルブラン領域は二つの定数$\{a,b\}$からなり，言語$L$は二つの引数を取る述語記号$blames$しか持たないとする．次の4つの基底アトム$\{blames(a,a), blames(a,b)$, $blames(b,a)$, $blames(b,b)\}$が存在する．今，世界が取り得る世界の状態（モデル）は$2^{4}=16$通り存在する．各モデルにおいて異なる真偽値の割り当て（エルブラン解釈）が与えられる．表\ref{tab:FOL}の各行は異なるモデルを表し，最後のセルにはそのモデルに属するデータが$\times$で記されている．データが存在しないモデルは省略されている．このとき次式が成立する．
\begin{eqnarray*}
&&p(\forall x~ blames(x,a)|\exists x~blames(x,a))\\
&&=\frac{\sum_{n=1}^{16}\llbracket\forall x~ blames(x,a),\exists x~blames(x,a)\rrbracket_{m_{n}}\frac{K_{n}}{K}}{\sum_{n=1}^{16}\llbracket\exists x~ blames(x,a)\rrbracket_{m_{n}}\frac{K_{n}}{K}}\\
&&=\frac{K_{2}}{K_{1}+K_{2}}=\frac{3}{2+3}=0.6.
%
%
%
%&&p(\exists x~blames(x,a))\\
%&&=\sum_{n=1}^{16}p(\exists x~ blames(x,a)|m_{n})\sum_{k=1}^{10}p(m_{n}|d_{k})p(d_{k})\\
%&&=\sum_{n=1}^{16}\llbracket\exists x~ blames(x,a)\rrbracket_{m_{n}}\frac{K_{n}}{10}=\frac{K_{1}+K_{2}}{10}=\frac{5}{10}\\
%
%&&p(\forall x~ blames(x,a))\\
%&&=\sum_{n=1}^{16}p(\forall x~ blames(x,a)|m_{n})\sum_{k=1}^{10}p(m_{n}|d_{k})p(d_{k})\\
%&&=\sum_{n=1}^{16}\llbracket\forall x~ blames(x,a)\rrbracket_{m_{n}}\frac{K_{n}}{10}=\frac{K_{2}}{10}=\frac{3}{10}\\
%
%&&p(\forall x\exists y~ blames(x,y))\\
%&&=\sum_{n=1}^{16}p(\forall x\exists y~ blames(x,y)|m_{n})\sum_{k=1}^{10}p(m_{n}|d_{k})p(d_{k})\\
%&&=\sum_{n=1}^{16}\llbracket\forall x\exists y~ blames(x,y)\rrbracket_{m_{n}}\frac{K_{n}}{10}=\frac{\sum_{k=1}^{3}K_{k}}{10}=1
\end{eqnarray*}
%%%%%%%%%
\begin{table}[t]
\caption{Three predicate models and ten associated data.}
\centering
\begin{tabular}{c|cccc|c}
 & \multicolumn{4}{c|}{$blames$} &\\
 & $(a,a)$ & $(a,b)$ & $(b,a)$ & $(b,b)$ & data ${\cal D}$\\\hline
$m_{1}$ & 1 & 0 & 0 & 1 & $\times\times$\\
$m_{2}$ & 1 & 1 & 1 & 0 & $\times\times\times$\\
$m_{3}$ & 0 & 1 & 0 & 1 & $\times\times\times\times\times$\\
other & \multicolumn{4}{c|}{other} & no data
\end{tabular}
\label{tab:FOL}
\end{table}
%%%%%%%%%
\end{example}
%%%%%%
%
%
%
%%%%%%%%%%%%%%%%%%%%%%%%%%%%%%%%%%%%%%%%%%%%%%%%%%%%%%%%%%%%%%%%%%%%%%%%%%%%%%%%%%%%%%%%%%%%%%%%%%%%%%%%%%%%%%%%%%%%%%%%%%%%%%%%%%%%%%%%%%%%%%%%%%%%%%%%%%%%%%%%%%%%%%%%%%%%%%%%%%%%%%%%%%%%%%%%%%%%%%%%%%%%%%%%%%%%%%%%%%%%%%%%%%%%%%%%%%%%%%%%%%%%%%%%%%%%%%%%%%%%%%%%%%%%%%%%%%%%%%%%%%%%%%%%
\section{Correctness}
\subsection{Statistical Estimation}
Fenstad \cite{fenstad:67} says that the probability of a formula is the sum of the probabilities of the models where the formula is true. Let $\alpha\in L$ and $m_{n}\in{\cal M}$. When $L$ has no function symbol or open formula, the first Fenstad theorem can have the following simpler form, where $m_{n}\models\alpha$ represents $m_{n}$ satisfies $\alpha$.
%$\alpha$はある論理式を表し，$m_{n}$は$n$番目のモデルを表すとする．Fenstad\cite{fenstad:67}は論理式$\alpha$の確率はそれを真とするモデルの確率の和と等しくなければならないという．For the sake of simplicity, 関数記号と開論理式は扱わない．このとき，Fenstadの最初の表現定理は次の単純な表現を持ちうる．
%
\begin{eqnarray}\label{eq:fenstad}
p(\alpha)=\sum_{n=1: m_{n}\models\alpha}^{N}p(m_{n})
\end{eqnarray}
%
%ここで$m\models\alpha$は$m$ satisfies $\alpha$を表す．$M$はモデルを表す確率変数であり，$p(M)$は$\Phi=(\phi_{1},\phi_{2},...,\phi_{K})$をパラメータに持つカテゴリカル分布である．$\sum_{k=1}^{K}\phi_{k}=1$である．$K$はモデルの数であり，$\phi_{k}$は$k$番目のモデル$m_{k}$が生起する確率を表す，つまり$p(M=m_{k})=\phi_{k}$である．
\par
When one has no prior knowledge about the probability of models, the most frequently used method to estimate $p(M)$ only from data is maximum likelihood estimation, which is given as follows.
%${\cal D}=\{d_{1},d_{2},...,d_{K}\}$はデータの集合（データセット）を表す．$K$はデータ数である．我々は${\cal D}$が離散的でありかつ完全であると仮定する．すなわち任意のデータ$d_{k}\in{\cal D}$はある唯一のモデルに属する．$K_{n}$でモデル$m_{n}$に属するデータの数を表す．モデルについての事前知識を持たないとき，データセットからモデルの事前確率$p(M)$を統計的に求める典型的な方法は最尤推定法である．それは次式で表される．
%
\begin{eqnarray*}
p(M)=\argmax_{\Phi}p({\cal D}|\Phi)
%&=&\argmax_{\bm{\phi}}\sum_{d}\log p(d|\Phi)
\end{eqnarray*}
Assuming that each data is independent given $\Phi$, we have
%パラメータ$\Phi$に従ってデータが独立に生成されると仮定すると，次式が成り立つ．
%
\begin{eqnarray*}
&&p({\cal D}|\Phi)=\prod_{k=1}^{K}p(d_{k}|\Phi)\\%=\phi_{1}^{N_{1}}\phi_{2}^{N_{2}}\cdots\phi_{K}^{N_{K}}\\
&&=\phi_{1}^{K_{1}}\phi_{2}^{K_{2}}\cdots\phi_{N-1}^{K_{N-1}}(1-\phi_{1}-\phi_{2}-\cdots-\phi_{N-1})^{K_{N}}.
%=\prod_{k=1}^{K}\phi_{k}^{\#{\cal D}_{k}}
\end{eqnarray*}
$\Phi$ maximises the likelihood if and only if it maximises the log likelihood, which is given as follows.
\begin{eqnarray*}
L(\Phi)&=&K_{1}\log\phi_{1}+K_{2}\log\phi_{2}+\cdots+K_{N-1}\log\phi_{N-1}\\
&&+K_{N}\log(1-\phi_{1}-\phi_{2}-\cdots-\phi_{N-1})
\end{eqnarray*}
The maximum likelihood estimate is obtained by solving the following simultaneous equations, which are obtained by differentiating the log likelihood with respect to each $\phi_{n}(1\leq n\leq N-1)$.
%最尤推定量は次の連立方程式を解くことによって得られる，それは対数尤度を各$\phi_{n}(1\leq n\leq N-1)$で微分することで得られる．
%
\begin{eqnarray*}
\frac{\partial L(\Phi)}{\partial \phi_{n}}=\frac{K_{n}}{\phi_{n}}-\frac{K_{N}}{1-\phi_{1}-\phi_{2}-\cdots-\phi_{N-1}}=0
\end{eqnarray*}
The following is the solution to the simultaneous equations.
%\footnote{代入すれば解であることは容易に確かめられる}．
%
\begin{eqnarray*}
\Phi=\left(\frac{K_{1}}{K},\frac{K_{2}}{K},...,\frac{K_{N}}{K}\right)
\end{eqnarray*}
Therefore, the maximum likelihood estimate for the $n$-th model is just the ratio of the number of data in the model to the total number of data. Combining Equation (\ref{eq:fenstad}) and the maximum likelihood estimate, we have
%つまり，最尤推定量$\phi_{n}\in\Phi$は全データのうち$n$番目のモデルに属するデータの割合である．従って，Fenstadの定理と最尤推定により次の結果を得る．
%
\begin{eqnarray}\label{eq:fenstad+ML}
p(\alpha)=\sum_{n=1: m_{n}\models\alpha}^{N}\frac{K_{n}}{K}.
\end{eqnarray}
\par
Now, let $\{p(\Delta|M, \mu=1), p(M|D), p(D)\}$ be a logical model such that $\mu=1$. We show that both the Fenstad theorem and maximum likelihood estimation justify the logical model. The Fenstad theorem justifies the logical model because probabilistic inference on the logical model satisfies Equation (\ref{eq:fenstad}).
%さて，我々はFenstadの定理と最尤推定法がどちらも我々の方法を正当化することを示す．Let $\{p(\Delta|M, \mu=1), p(M|D), p(D)\}$とする．Fenstadの定理は我々の方法を正当化する，なぜなら我々の生成モデル上の確率推論は式（\ref{eq:fenstad}）に至るからである．
%
\begin{eqnarray*}
p(\alpha)&=&\sum_{n=1}^{N}p(\alpha,m_{n})=\sum_{n=1}^{N}p(\alpha|m_{n})p(m_{n})\\
&=&\sum_{n=1}^{N}\llbracket\alpha\rrbracket_{m_{n}}p(m_{n})=\sum_{n=1:m_{n}\in\llbracket\alpha\rrbracket}^{N}p(m_{n})
\end{eqnarray*}
Maximum likelihood estimation also justifies the logical model because probabilistic inference on the logical model satisfies Equation (\ref{eq:fenstad+ML}).
%また，最尤推定法も我々のアプローチを正当化する．なぜなら我々の生成モデル上の確率推論は式（\ref{eq:fenstad+ML}）に至るからである．
%
\begin{eqnarray}\label{eq:data-driven}
p(\alpha)&=&\sum_{n=1}^{N}\sum_{k=1}^{K}p(\alpha,m_{n},d_{k})\nonumber\\
&=&\sum_{n=1}^{N}p(\alpha|m_{n})\sum_{k=1}^{K}p(m_{n}|d_{k})p(d_{k})\nonumber\\
&=&\sum_{n=1}^{N}\llbracket\alpha\rrbracket_{m_{n}}\frac{K_{n}}{K}=\sum_{n=1: m_{n}\in\llbracket\alpha\rrbracket}^{N}\frac{K_{n}}{K}
%&=&\sum_{n=1}^{N}\llbracket\alpha\rrbracket_{m_{n}}p(m_{n})=\sum_{n=1:m_{n}\in\llbracket\alpha\rrbracket}^{N}p(m_{n})\\\\
%p(\alpha)&=&\sum_{n=1}^{N}p(\alpha|m_{n})p(m_{n})\nonumber\\
%&=&\sum_{k=1}^{K}p(\alpha|m(d_{k}))p(m(d_{k}))\nonumber\\
%&=&\frac{1}{K}\sum_{k=1}^{K}\llbracket\alpha\rrbracket_{m(d_{k})}\nonumber\\
%&=&\frac{1}{K}\left\{\llbracket\alpha\rrbracket_{m(d_{1})}+\llbracket\alpha\rrbracket_{m(d_{2})}+\cdots+\llbracket\alpha\rrbracket_{m(d_{K})}\right\}\nonumber\\
%&=&\frac{1}{K}\left\{K_{1}\llbracket\alpha\rrbracket_{m_{1}}+K_{2}\llbracket\alpha\rrbracket_{m_{2}}+\cdots+K_{N}\llbracket\alpha\rrbracket_{m_{N}}\right\}\nonumber\\
%&=&\frac{1}{K}\sum_{n=1}^{N}K_{n}\llbracket\alpha\rrbracket_{m_{n}}=\sum_{n=1: m_{n}\in\llbracket\alpha\rrbracket}^{N}\frac{K_{n}}{K}\\
\end{eqnarray}
%
%今，任意の$k$に関して，$p(m(d_{k}))=\frac{1}{K}$を仮定する．すなわちデータが属するモデルは等確率で生起する．次式はこの確率推論が式（\ref{eq:fenstad+ML}）を満たすことを示す．
%
%\begin{eqnarray*}
%p(\alpha)&=&\sum_{k=1}^{K}p(\alpha|m(d_{k}))p(m(d_{k}))=\frac{1}{K}\sum_{k=1}^{K}\llbracket\alpha\rrbracket_{m(d_{k})}\\
%&=&\frac{1}{K}\left\{\llbracket\alpha\rrbracket_{m(d_{1})}+\llbracket\alpha\rrbracket_{m(d_{2})}+\cdots+\llbracket\alpha\rrbracket_{m(d_{K})}\right\}\\
%&=&\frac{1}{K}\left\{K_{1}\llbracket\alpha\rrbracket_{m_{1}}+K_{2}\llbracket\alpha\rrbracket_{m_{2}}+\cdots+K_{N}\llbracket\alpha\rrbracket_{m_{N}}\right\}\\
%&=&\frac{1}{K}\sum_{n=1}^{N}K_{n}\llbracket\alpha\rrbracket_{m_{n}}=\sum_{n=1: m_{n}\in\llbracket\alpha\rrbracket}^{N}\frac{K_{n}}{K}
%\end{eqnarray*}
%
\par
%In the above discussion, the Fenstad theorem concerns the relation between formulae and models, whereas maximum likelihood estimation concerns the relation between models and data. We have shown that the logical model not only follows the concepts but also treats them as probabilistic inference in a unified way.
We have shown that the logical model not only follows the Fenstad theorem and maximum likelihood estimation but also treats their results as probabilistic inference in a unified way. Their results are both in the scope of statistics shown in Figure \ref{fig:LandS}. This is an important fact because, in the next section, we will discuss that the logical model can also deal with the scope of logic shown in Figure \ref{fig:LandS}. 
%
%Fenstadの定理は論理式に適切な確率値を割り当てるためにモデルを参照し，最尤推定法はモデルに適切な確率値を割り当てるためにデータを参照する．我々の生成モデル上の確率的推論はこれらを代替している．
%%%
%
%\par
%我々は我々の生成モデルがFenstadの定理及び最尤推定量をfollowすることを示した，これらはいずれも図\ref{fig:fig1}で示すstatisticsの範疇である．これは極めて重要な事実である，なぜなら我々は次章においてそれがlogicの範疇をカバーできることをも議論するからである．
\par
There are some practical advantages of the logical models. The computational complexity of Equation (\ref{eq:data-driven}) depends on $N$, which is unbounded in predicate logic and exponentially increases in propositional logic with respect to the number of propositional symbols. However, Equation (\ref{eq:data-driven}) can be transformed as follows for a linear complexity with respect to the number of data, i.e., $K$.
\begin{eqnarray}\label{eq:linear}
p(\alpha)=\sum_{n=1}^{N}\llbracket\alpha\rrbracket_{m_{n}}\frac{K_{n}}{K}=\sum_{k=1}^{K}\llbracket\alpha\rrbracket_{m(d_{k})}\frac{1}{K}
\end{eqnarray}
In addition, Equation (\ref{eq:data-driven}) has only a constant complexity for recalculation for new data. Let $p_{K}$ denote the probability calculated with $K$ data. $p_{K+1}(\alpha)$ can be calculated using $p_{K}(\alpha)$ as follows.
%
%式（\ref{eq:data-driven}）の計算量はモデル数$K$に依存する，それは述語論理ではunboundedであり，命題論理では命題記号数に関して指数関数的に増える．幸運なことに，これは変形できる，データ数に関して線形の計算量を持つように．
%
%\begin{eqnarray*}
%p(\alpha)=\sum_{n=1}^{N}\llbracket\alpha\rrbracket_{m_{n}}\frac{K_{n}}{K}=\sum_{k=1}^{K}\llbracket\alpha\rrbracket_{m(d_{k})}\frac{1}{K}
%\end{eqnarray*}
%これは式（\ref{eq:fenstad+ML}）よりもずっとよい．なぜならその計算量はモデル数$K$に依存する，それは述語論理ではunboundedであり，命題論理では命題記号数に関して指数関数的に増える．
%
%さらに，新たなデータが獲得されるとき，式（\ref{eq:data-driven}）の再計算は定数の計算量しか持たない．$K$個のデータから計算された確率を$p_{K}$と書く．$p_{K+1}(\alpha)$は$p_{K}(\alpha)$から次式で計算される．
%
\begin{eqnarray}\label{eq:updating}
p_{K+1}(\alpha)&=&\sum_{n=1}^{N}p(\alpha|m_{n})\sum_{k=1}^{K+1}p(m_{n}|d_{k})p(d_{k})\nonumber\\
&=&\sum_{n=1}^{N}p(\alpha|m_{n})\sum_{k=1}^{K}p(m_{n}|d_{k})p(d_{k})\nonumber\\
&&+\sum_{n=1}^{N}p(\alpha|m_{n})p(m_{n}|d_{K+1})p(d_{K+1})\nonumber\\
&=&\frac{K}{K+1}\sum_{n=1}^{N}p(\alpha|m_{n})\sum_{k=1}^{K}p(m_{n}|d_{k})\frac{1}{K}\nonumber\\
&&+\sum_{n=1}^{N}p(\alpha|m_{n})p(m_{n}|d_{K+1})\frac{1}{K+1}\nonumber\\
&=&\frac{Kp_{K}(\alpha)+\llbracket\alpha\rrbracket_{m(d_{K+1})}}{K+1}
%&=\frac{\llbracket\alpha\rrbracket_{m(d_{N+1})}+Np_{N}(\alpha)}{N+1}
\end{eqnarray}
\par
Finally, as demonstrated in the following example, Equation (\ref{eq:updating}) is good at modelling the development of commonsense knowledge.
%最後に，式（\ref{eq:data-driven}）はデータドリブンな常識の扱いに優れている．次の例は，我々の生成モデルはデータが知識を更新して常識を獲得する過程を確率的にモデル化していることを示す．
%
%%%%%%%%%%%%%%%
\begin{example}
\begin{table}[t]
\centering
\setlength{\tabcolsep}{0.5mm} 
\caption{New data.}
\label{ex:update}
\begin{tabular}{c|cc|c|c}
& $bird$ & $fly$ & data & new data\\\hline
$m_{1}$ & $0$ & $0$ & $\times\times\times\times\times$ &\\
$m_{2}$ & $0$ & $1$ & $\times\times$ &\\
$m_{3}$ & $1$ & $0$ & & $\times$\\
$m_{4}$ & $1$ & $1$ & $\times\times\times$
\end{tabular}
\end{table}
%
%Each cell of Table \ref{tab:PL} corresponds to a different model constructed with propositional symbols `$bird$' and `$fly$' meaning that `it is a bird' and `it flies,' respectively. Each cell has data that belongs to the corresponding models.
%
Let `$bird$' and `$fly$' be two propositional symbols meaning `It is a bird.' and `It flies.', respectively. Each row of Table \ref{ex:update} shows a different model. Given the ten data shown in the fourth column, the probability that $bird$ implies $fly$ is calculated using Equation (\ref{eq:linear}), as follows.
%\par
%次式は含意「$bird\rightarrow fly$」の確率を示す．
%
\begin{align*}
%&p(fly\leftarrow bird)=\sum_{n=1}^{4}p(fly\leftarrow bird|m_{n})\sum_{k=1}^{10}p(m_{n}|d_{k})p(d_{k})\\
&p(bird\rightarrow fly)=\sum_{k=1}^{10}\llbracket bird\rightarrow fly\rrbracket_{m(d_{k})}\frac{1}{10}=1
%&=\sum_{n=1}^{4}\llbracket fly\leftarrow bird\rrbracket_{m_{n}}\frac{K_{n}}{10}\\
%&=\frac{\sum_{n=1}^{10}\llbracket fly\leftarrow bird\rrbracket_{m(d_{n})}}{10}=1
%=\frac{|\llbracket fly\leftarrow bird\rrbracket|}{10}
\end{align*}
It is obvious from the logical model that the counterintuitive knowledge that birds must fly comes from a lack of data. Indeed, taking into account the eleventh data shown in the last column, the probability is updated using Equation (\ref{eq:updating}), as follows.
%「鳥は必ず飛ぶ」というこの推論結果は常識に反する．これがデータ不足によって生じていることは我々の推論メカニズムから明らかである．事実，モデル$m_{3}$に属する11個目のデータの観測はこの結果を否定する．その更新は式（\ref{eq:updating}）に従う．
%
\begin{eqnarray*}
p_{11}(\alpha)=\frac{10p_{10}(bird\rightarrow fly)+\llbracket bird\rightarrow fly\rrbracket_{m(d_{11})}}{11}=\frac{10}{11}
\end{eqnarray*}
\end{example}
%
%
%
%
%%%%%%%%%%%%%%%%%%%%%%%%%%%%%%%%%%%%%%%%%%%%%%%%%%%%%%%%%%%%%%%%%%%%%%%%%%%%%%%%%%%%%%%%%%%%%%%%%%%%%%%%%%%%%%%%%%%%%%%%%%%%%%%%%%%%%%%%%%%%%%%%%%%%%%%%%%%%%%%%%%%%%%%%%%%%%%%%%%%%%%%%%%%%%%%%%%%%%%%%%%%%%%%%%%%%%%%%%%%%%%%%%%%%%%%%%%%%%%%%%%%%%%%%%%%%%%%%%%%%%%%%%%%%%%%%%%%%%%%%%%%%%%%%
%\section{Logical Entailment}
%
%%%%%%%%%%%%%%%%%%%%%%%%%%%%%%%%%%%%%%%%%%%%%%%%%%%%%%%%%%%%%%%%%%%%%%%%%%%%%%%%%%%%%%%%%%%%%%%%%%%%%%%%%%%%%%%%%%%%%%%%
\subsection{Logical Entailment}
%In this section, the logical model is specialised in several ways to show its correctness in terms of the logics of paraconsistency, counterfactuals, nonmonotonicity and prediction.
%\par
%
We showed in the last section that, given $\{p(\Delta|M,\mu=1),p(M|D),p(D)\}$, $p(M)$ is equivalent to the maximum likelihood estimate, i.e., for all $m_{n}\in{\cal M}$,
%この論理モデルが与えられる時，我々は前章において$p(M)$は最尤推定量に等しいことを議論した．すなわち任意のモデル$m_{n}\in{\cal M}$において次式が成り立つことを示した．
%
\begin{eqnarray*}
p(m_{n})=\sum_{k=1}^{K}p(m_{n}|d_{k})p(d_{k})=\frac{K_{n}}{K}.
\end{eqnarray*}
%
%In other words, we showed the following equation, for all models $m_{n}\in{\cal M}$.
Therefore, $\{p(\Delta|M,\mu=1),p(M|D),p(D)\}$ is equivalent to $\{p(\Delta|M,\mu=1),p(M)\}$ when $p(M)$ is the maximum likelihood estimate. For the sake of simplicity, we also call the latter a logical model and use it without distinction. To discuss logical properties of the logical model, we assume $0\notin p(M)$ meaning that every model is possible, i.e., $p(m)\neq 0$, for all models. Recall that a set $\Delta$ of formulae entails a formula $\alpha$ in classical logic, denoted by $\Delta\models\alpha$, iff $\alpha$ is true in every model in which $\Delta$ is true. The following two theorems state that certain inference on the logical model is more cautious than classical entailment.
%
%すなわち$\{p(\Delta|M,\mu=1),p(M|D),p(D)\}$と最尤推定量で定義されたモデルの事前確率を持つ$p(M)$を持つような論理モデル$\{p(\Delta|M,\mu=1),p(M)\}$は等価である．議論を簡単にするために本節ではこの後者の論理モデル$\{p(\Delta|M,\mu=1),p(M)\}$を用いる．
%and $0\notin\bm{\phi}$, where $0\notin\bm{\phi}$ represents $\phi\neq 0$ for all $\phi\in\bm{\phi}$. We call the Bayesian entailment defined on the logical model a Bayesian classical entailment. This is a Bayesian entailment with the assumptions that there is no noise in the observation of $\Delta$ and that there is no impossible state of the world. 
%
\begin{theorem}\label{thrm:1}
%Let $\alpha\in L$, $\Delta\subseteq L$ and $\vapprox_{1}$ be the Bayesian classical entailment. If there is a model of $\Delta$ then $\Delta\vapprox_{1}\alpha$ iff $\Delta\models\alpha$.
%
Let $\alpha\in L$ and $\Delta\subseteq L$ such that $\llbracket\Delta\rrbracket\neq\emptyset$. $p(\alpha|\Delta)=1$ if and only if $\Delta\models\alpha$.
\end{theorem}
\begin{proof}
Recall that, in formal logic, the fact that there is a model of $\Delta$ (or $\Delta$ has a model) is equivalent to the fact that there is a model $m$ in which every formula in $\Delta$ is true in $m$. Dividing models into the models of $\Delta$ and the others, we have
\begin{align*}
&p(\alpha|\Delta)=\frac{\sum_{m}p(\alpha|m)p(\Delta|m)p(m)}{\sum_{m}p(\Delta|m)p(m)}\\
&=\frac{\displaystyle{\sum_{m\in\llbracket\Delta\rrbracket}p(m)p(\alpha|m)\mu^{|\Delta|}+\sum_{m\notin\llbracket\Delta\rrbracket}p(m)p(\alpha|m)p(\Delta|m)}}{\displaystyle{\sum_{m\in\llbracket\Delta\rrbracket}p(m)\mu^{|\Delta|}+\sum_{m\notin\llbracket\Delta\rrbracket}p(m)p(\Delta|m)}}.
\end{align*}
$p(\Delta|m)=\prod_{\beta\in\Delta}p(\beta|m)=\prod_{\beta\in\Delta}\mu^{\llbracket\beta\rrbracket_{m}}(1-\mu)^{1-{\llbracket\beta\rrbracket_{m}}}$. For all $m\notin\llbracket\Delta\rrbracket$, there is $\beta\in\Delta$ such that $\llbracket\beta\rrbracket_{m}=0$. Therefore, $p(\Delta|m)=0$ when $\mu=1$, for all $m\notin\llbracket\Delta\rrbracket$. We thus have
\begin{align*}
p(\alpha|\Delta)=&\frac{\sum_{m\in\llbracket\Delta\rrbracket}p(m)p(\alpha|m)1^{|\Delta|}}{\sum_{m\in\llbracket\Delta\rrbracket}p(m)1^{|\Delta|}}\\
=&\frac{\sum_{m\in\llbracket\Delta\rrbracket}p(m)1^{\llbracket\alpha\rrbracket_{m}}0^{1-\llbracket\alpha\rrbracket_{m}}}{\sum_{m\in\llbracket\Delta\rrbracket}p(m)}.
\end{align*}
Since $1^{\llbracket\alpha\rrbracket_{m}}0^{1-\llbracket\alpha\rrbracket_{m}}=1^{1}0^{0}=1$ if $m\in\llbracket\alpha\rrbracket$ and $1^{\llbracket\alpha\rrbracket_{m}}0^{1-\llbracket\alpha\rrbracket_{m}}=1^{0}0^{1}=0$ if $m\notin\llbracket\alpha\rrbracket$, we have
\begin{align*}
p(\alpha|\Delta)=\frac{\sum_{m\in\llbracket\Delta\rrbracket\cap\llbracket\alpha\rrbracket}p(m)}{\sum_{m\in\llbracket\Delta\rrbracket}p(m)}.
\end{align*}
Now, $\frac{\sum_{m\in\llbracket\Delta\rrbracket\cap\llbracket\alpha\rrbracket}p(m)}{\sum_{m\in\llbracket\Delta\rrbracket}p(m)}=1$ iff $\llbracket\alpha\rrbracket\supseteq\llbracket\Delta\rrbracket$, i.e., $\Delta\models\alpha$.
\end{proof}
\begin{example}
Theorem \ref{thrm:1} does not hold without assumption $0\notin p(M)$. Given $p(M)=(0.6,0,0.1,0.3)$ in Example \ref{ex:BE}, $p(rain|wet)=1$ but $\{wet\}\not\models rain$.
\end{example}
\begin{theorem}\label{thrm:2}
Let $\alpha\in L$ and $\Delta\subseteq L$ such that $\llbracket\Delta\rrbracket=\emptyset$. If $p(\alpha|\Delta)=1$ then $\Delta\models\alpha$, but not vice versa.
\end{theorem}
\begin{proof}
($\Rightarrow$) If $\llbracket\Delta\rrbracket=\emptyset$ then $\Delta\models\alpha$, for all $\alpha$, in classical logic. ($\Leftarrow$) We show a counterexample where $\Delta\models\alpha$ but $p(\alpha|\Delta)$ is undefined. $\beta,\lnot\beta\models\alpha$ holds because $\llbracket\beta,\lnot\beta\rrbracket=\emptyset$ results in $\llbracket\beta,\lnot\beta\rrbracket\subseteq\llbracket\alpha\rrbracket$. Meanwhile, $p(\alpha|\beta,\lnot\beta)$ is given as follows.
\begin{align*}
&p(\alpha|\beta,\lnot\beta)=\frac{\sum_{w}p(w)p(\alpha|w)p(\beta|w)p(\lnot\beta|w)}{\sum_{w}p(w)p(\beta|w)p(\lnot\beta|w)}\\
&=\frac{\mu(1-\mu)\sum_{w}p(w)p(\alpha|w)}{\mu(1-\mu)\sum_{w}p(w)}
\end{align*}
This is undefined due to division by zero when $\mu=1$.
\end{proof}
\par
Everything is entailed from a contradiction in the classical entailment. Certain inference on the logical model is more cautious than the classical entailment because the proof of Theorem \ref{thrm:2} states that nothing is entailed from a contradiction. In the next section, we look at a logical model that entails something reasonable from contradictions.
%However, the proof of Theorem \ref{thrm:2} implies that nothing is entailed from contradictions in the Bayesian classical entailment. 
%
%%%%%%%%%%%%%%%%%%%%%%%%%%%%%%%%%%%%%%%%%%%%%%%%%%%%%%%%%%%%%%%%%%%%%%%%%%%%%%%%%%%%%%%%%%%%%%%%%%%%%%%%%%%%%%%%%%%%%%%%
\subsection{Paraconsistency}
Let $\{\lim_{\mu\rightarrow 1}p(\Delta|M,\mu), p(M)\}$ be a logical model such that $\mu\rightarrow 1$ and $0\notin p(M)$ where $\mu\rightarrow 1$ represents $\mu$ approaches 1. The following two theorems state that certain inference on the logical model is more cautious than classical entailment.
\begin{theorem}\label{thrm:3}
Let $\alpha\in L$ and $\Delta\subseteq L$ such that $\llbracket\Delta\rrbracket\neq\emptyset$. $p(\alpha|\Delta)=1$ if and only if $\Delta\models\alpha$.
\end{theorem}
\begin{proof}
%The proof of Theorem \ref{thrm:1} still holds under the presence of the limit operation.
$\lim_{\mu\rightarrow 1}$ does not change the proof of Theorem \ref{thrm:1}.
\end{proof}
\begin{theorem}\label{thrm:4}
Let $\alpha\in L$ and $\Delta\subseteq L$ such that $\llbracket\Delta\rrbracket=\emptyset$. If $p(\alpha|\Delta)=1$ then $\Delta\models\alpha$, but not vice versa.
\end{theorem}
\begin{proof}
($\Rightarrow$) Same as for Theorem 2. ($\Leftarrow$) We show a counterexample where $\Delta\models\alpha$ but $p(\alpha|\Delta)\neq 1$. Suppose $p(\alpha)<1$. We can show $p(\alpha|\beta\land\lnot\beta)<1$ as follows.
%The following derivation exemplifies $p(\alpha|\beta\land\lnot\beta)<1$.
%
\begin{align*}
&p(\alpha|\beta\land\lnot\beta)\\
&=\frac{\sum_{m}p(m)\lim_{\mu\rightarrow 1}p(\alpha|m)\lim_{\mu\rightarrow 1}p(\beta\land\lnot\beta|m)}{\sum_{m}p(m)\lim_{\mu\rightarrow 1}p(\beta\land\lnot\beta|m)}\\
&=\lim_{\mu\rightarrow 1}\frac{(1-\mu)\sum_{m}p(m)p(\alpha|m)}{(1-\mu)\sum_{m}p(m)}=\lim_{\mu\rightarrow 1}\frac{\sum_{m}p(m)p(\alpha|m)}{\sum_{m}p(m)}\\
&=\sum_{m}p(m)\lim_{\mu\rightarrow 1}p(\alpha|m)=p(\alpha)
\end{align*}
Therefore, $p(\alpha|\beta\land\lnot\beta)\neq 1$. Note that $\beta\land\lnot\beta\models\alpha$ because $\llbracket\beta\land\lnot\beta\rrbracket=\emptyset$ results in $\llbracket\beta\land\lnot\beta\rrbracket\subseteq\llbracket\alpha\rrbracket$.
%See Appendices.
%In contrast to Theorem \ref{thrm:2}, 
\end{proof}
\par
To characterise the certain inference on the logical model, we define an approximate model using maximal consistent subsets with respect to set cardinality. Recall that a set of formulae is consistent if there is a model of the set. 
%
%極大無矛盾部分集合を用いた矛盾からの推論方法の観点から見るベイジアン准無矛盾伴意の妥当性を示す．Recall that a set of formulae is consistent if there is a model of the set. 集合の要素数に関して極大な無矛盾部分集合（set-cardinality-maximal consistent subsets）を使って近似モデルを定義する．これはNeurIPSで採用した定義とは異なるし，また集合の包含関係に関して極大な無矛盾部分集合（set-inclusion-maximal consistent subsets）を使った定義とも異なる．
%
\begin{definition}[Approximate model]\label{def:approximatemodel}
Let $m$ be a model and $\Delta\subseteq L$ be an inconsistent set of formulae. $m$ is an approximate model of $\Delta$ if $m$ is a model of a maximal (w.r.t. set cardinality) consistent subset of $\Delta$.
\end{definition}
\begin{theorem}\label{th1}
Let $\Delta\subseteq L$ and $\alpha\in L$. $p(\alpha|\Delta)=1$ if and only if $\Delta'\models\alpha$, for all maximal (w.r.t. set cardinality) consistent subsets $\Delta'$ of $\Delta$.
\end{theorem}
\begin{proof}
We use notation $(\!(\Delta)\!)$ to denote the set of all approximate models of $\Delta$. We also use notation $|
\Delta|_{m}$ to denote the number of formulas in $\Delta$ that are true in $m$, i.e. $|\Delta|_{m}=\sum_{\beta\in\Delta}\llbracket\beta\rrbracket_{m}$. Dividing models into $(\!(\Delta)\!)$ and the others, we have 
\begin{align*}
&p(\alpha|\Delta)=\lim_{\mu\rightarrow 1}\frac{\sum_{m}p(\alpha|m)p(m)p(\Delta|m)}{\sum_{m}p(m)p(\Delta|m)}=\lim_{\mu\rightarrow 1}\\
&\frac{\displaystyle{\sum_{\hat{m}\in(\!(\Delta)\!)}p(\alpha|\hat{m})p(\hat{m})p(\Delta|\hat{m})+\sum_{m\notin(\!(\Delta)\!)}p(\alpha|m)p(m)p(\Delta|m)}}{\displaystyle{\sum_{\hat{m}\in(\!(\Delta)\!)}p(\hat{m})p(\Delta|\hat{m})+\sum_{m\notin(\!(\Delta)\!)}p(m)p(\Delta|m)}}.
\end{align*}
%\end{eqnarray*}
%
Now, $p(\Delta|m)$ can be developed as follows, for all $m$ (regardless of the membership of $(\!(\Delta)\!)$).
\begin{align*}
&p(\Delta|m)=\prod_{\beta\in\Delta}p(\beta|m)=\prod_{\beta\in\Delta}\mu^{\llbracket\beta\rrbracket_{m}}(1-\mu)^{1-\llbracket\beta\rrbracket_{m}}\\
&=\mu^{\sum_{\beta\in\Delta}\llbracket\beta\rrbracket_{m}}(1-\mu)^{\sum_{\beta\in\Delta}(1-\llbracket\beta\rrbracket_{m})}\\
&=\mu^{|\Delta|_{m}}(1-\mu)^{|\Delta|-|\Delta|_{m}}
\end{align*}
Therefore, $p(\alpha|\Delta)=\lim_{\mu\rightarrow 1}\frac{W+X}{Y+Z}$ where
\begin{align*}
&W=\sum_{\hat{m}\in(\!(\Delta)\!)}p(\alpha|\hat{m})p(\hat{m})\mu^{|\Delta|_{\hat{m}}}(1-\mu)^{|\Delta|-|\Delta|_{\hat{m}}}\\
&X=\sum_{m\notin(\!(\Delta)\!)}p(\alpha|m)p(m)\mu^{|\Delta|_{m}}(1-\mu)^{|\Delta|-|\Delta|_{m}}\\
&Y=\sum_{\hat{m}\in(\!(\Delta)\!)}p(\hat{m})\mu^{|\Delta|_{\hat{m}}}(1-\mu)^{|\Delta|-|\Delta|_{\hat{m}}}\\
&Z=\sum_{m\notin(\!(\Delta)\!)}p(m)\mu^{|\Delta|_{m}}(1-\mu)^{|\Delta|-|\Delta|_{m}}
\end{align*}
%
%\begin{align*}
%\lim_{\mu\rightarrow 1}\frac{\displaystyle{\sum_{\hat{w}\in(\!(\Delta)\!)}p(\alpha|\hat{w})p(\hat{w})\mu^{|\Delta|_{\hat{w}}}(1-\mu)^{|\Delta|-|\Delta|_{\hat{w}}}+\sum_{w\notin(\!(\Delta)\!)}p(\alpha|w)p(w)\mu^{|\Delta|_{w}}(1-\mu)^{|\Delta|-|\Delta|_{w}}}}{\displaystyle{\sum_{\hat{w}\in(\!(\Delta)\!)}p(\hat{w})\mu^{|\Delta|_{\hat{w}}}(1-\mu)^{|\Delta|-|\Delta|_{\hat{w}}}+\sum_{w\notin(\!(\Delta)\!)}p(w)\mu^{|\Delta|_{w}}(1-\mu)^{|\Delta|-|\Delta|_{w}}}}.
%\end{align*}
%
%%Besides, $|\Delta|_{\hat{w}}> |\Delta|_{w}$ and $|\Delta|-|\Delta|_{w}> |\Delta|-|\Delta|_{\hat{w}}$, for all $\hat{w}\in(\!(\Delta)\!)$ and $w\notin(\!(\Delta)\!)$.
%
From Definition \ref{def:approximatemodel}, $|\Delta|_{\hat{m}}$ has the same value, for all $\hat{m}\in(\!(\Delta)\!)$. Therefore, the fraction can be simplified by dividing the denominator and numerator by $(1-\mu)^{|\Delta|-|\Delta|_{\hat{m}}}$. We thus have $p(\alpha|\Delta)=\lim_{\mu\rightarrow 1}\frac{W'+X'}{Y'+Z'}$ where
\begin{align*}
&W'=\sum_{\hat{m}\in(\!(\Delta)\!)}p(\alpha|\hat{m})p(\hat{m})\mu^{|\Delta|_{\hat{m}}}\\
&X'=\sum_{m\notin(\!(\Delta)\!)}p(\alpha|m)p(m)\mu^{|\Delta|_{m}}(1-\mu)^{|\Delta|_{\hat{m}}-|\Delta|_{m}}\\
&Y'=\sum_{\hat{m}\in(\!(\Delta)\!)}p(\hat{m})\mu^{|\Delta|_{\hat{m}}}\\
&Z'=\sum_{m\notin(\!(\Delta)\!)}p(m)\mu^{|\Delta|_{m}}(1-\mu)^{|\Delta|_{\hat{m}}-|\Delta|_{m}}.
\end{align*}
%
%\begin{align*}
%&p(\alpha|\Delta)\\
%&=\lim_{\mu\rightarrow 1}\frac{\displaystyle{\sum_{\hat{w}\in(\!(\Delta)\!)}p(\alpha|\hat{w})p(\hat{w})\mu^{|\Delta|_{\hat{w}}}+\sum_{w\notin(\!(\Delta)\!)}p(\alpha|w)p(w)\mu^{|\Delta|_{w}}(1-\mu)^{|\Delta|_{\hat{w}}-|\Delta|_{w}}}}{\displaystyle{\sum_{\hat{w}\in(\!(\Delta)\!)}p(\hat{w})\mu^{|\Delta|_{\hat{w}}}+\sum_{w\notin(\!(\Delta)\!)}p(w)\mu^{|\Delta|_{w}}(1-\mu)^{|\Delta|_{\hat{w}}-|\Delta|_{w}}}}.
%\end{align*}
%
Applying the limit operation, we have
\begin{align*}
p(\alpha|\Delta)=\frac{\displaystyle{\sum_{\hat{m}\in(\!(\Delta)\!)}p(\alpha|\hat{m})p(\hat{m})}}{\displaystyle{\sum_{\hat{m}\in(\!(\Delta)\!)}p(\hat{m})}}=\frac{\displaystyle{\sum_{\hat{m}\in(\!(\Delta)\!)}1^{\llbracket\alpha\rrbracket_{\hat{m}}}0^{1-\llbracket\alpha\rrbracket_{\hat{m}}}p(\hat{m})}}{\displaystyle{\sum_{\hat{m}\in(\!(\Delta)\!)}p(\hat{m})}}
\end{align*}
Since $1^{\llbracket\alpha\rrbracket_{\hat{m}}}0^{1-\llbracket\alpha\rrbracket_{\hat{m}}}=1^{1}0^{0}=1$ if $\hat{m}\in\llbracket\alpha\rrbracket$ and $1^{\llbracket\alpha\rrbracket_{\hat{m}}}0^{1-\llbracket\alpha\rrbracket_{\hat{m}}}=1^{0}0^{1}=0$ if $\hat{m}\notin\llbracket\alpha\rrbracket$, we have
\begin{align*}
p(\alpha|\Delta)=\frac{\sum_{\hat{m}\in(\!(\Delta)\!)\cap\llbracket\alpha\rrbracket}p(\hat{m})}{\sum_{\hat{m}\in(\!(\Delta)\!)}p(\hat{m})}.
\end{align*}
Therefore, $p(\alpha|\Delta)=1$ holds iff $\llbracket\alpha\rrbracket\supseteq (\!(\Delta)\!)$. By definition, $m\in (\!(\Delta)\!)$ iff $m$ is a model of a maximal consistent subset of $\Delta$  w.r.t. set cardinality. Therefore, $m\in (\!(\Delta)\!)$ iff $m\in\bigcup_{\Delta'}\llbracket\Delta'\rrbracket$ where $\Delta'$ is a maximal consistent subset of $\Delta$ w.r.t. set cardinality. Therefore, $p(\alpha|\Delta)=1$ iff $\llbracket\alpha\rrbracket\supseteq\bigcup_{\Delta'}\llbracket\Delta'\rrbracket$. In other words, for all maximal (w.r.t. set cardinality) consistent subsets $\Delta'$ of $\Delta$, $\llbracket\alpha\rrbracket\supseteq\llbracket\Delta'\rrbracket$, i.e., $\Delta'\models\alpha$.
\end{proof}
%
%
%%%%%%%%%
\begin{example}
%Suppose the uniform prior distribution over states of the world of two propositional variables $a$ and $b$. Since $p(a|a,b,\lnot b)=1$, $a,b,\lnot b\vapprox_{\theta}a$ iff $\theta=1$. Since $p(a|a\land b,\lnot b)=\frac{2}{3}$, $a\land b,\lnot b\vapprox_{\theta}a$ iff $\theta\leq \frac{2}{3}$. Since $p(a|a\land b\land\lnot b)=\frac{1}{2}$, $a\land b\land\lnot b\vapprox_{\theta}a$ iff $\theta\leq \frac{1}{2}$.
%
%Let us abbreviate $rain$ and $wet$ to $a$ and $b$, respectively. Given $\mu\rightarrow 1$ and $\bm{\phi}=(0.4,0.2,0.1,0.3)$ in Example \ref{ex:BE}, we have
Let $\mu\rightarrow 1$ and $p(M)=(0.25,0.25,0.25,0.25)$ in Example \ref{ex:BE}. Given $\Delta=\{rain,wet,rain\rightarrow wet,\lnot wet\}$, there are three maximal (w.r.t. set inclusion) consistent subsets, i.e., $S_{1}=\{rain,wet,rain\rightarrow wet\}$, $S_{2}=\{rain,\lnot wet\}$ and $S_{3}=\{rain\rightarrow wet,\lnot wet\}$, and one maximal (w.r.t. set cardinality) consistent subset, i.e., $S_{1}$. $p(rain|\Delta)=1$ and $S_{1}\models rain$ hold, but $S_{3}\not\models rain$.
%We have
%
%\begin{align*}
%&p(rain|\Delta)=\lim_{\mu\rightarrow 1}\\
%&=\frac{\sum_{m}p(rain|m)^{2}p(wet|m)p(rain\rightarrow wet|m)p(\lnot wet|m)p(m)}{\sum_{m}p(rain|m)p(wet|m)p(rain\rightarrow wet|m)p(m)}\\
%&p(rain|S_{1})\\
%&=\lim_{\mu\rightarrow 1}\frac{\sum_{m}p(rain|m)^{2}p(wet|m)p(rain\rightarrow wet|m)p(m)}{\sum_{m}p(rain|m)p(wet|m)p(rain\rightarrow wet|m)p(m)}\\
%&=\lim_{\mu\rightarrow 1}\frac{\mu(1-\mu)^{3}\phi_{1}+\mu(1-\mu)^{3}\phi_{2}+\mu^{3}(1-\mu)\phi_{3}+\mu^{3}(1-\mu)\phi_{4}}{\mu(1-\mu)^{2}\phi_{1}+\mu(1-\mu)^{2}\phi_{2}+\mu^{2}(1-\mu)\phi_{3}+\mu^{2}(1-\mu)\phi_{4}}\\
%
%&=\lim_{\mu\rightarrow 1}\frac{\mu(1-\mu)^{3}+\mu^{2}(1-\mu)^{2}+\mu^{2}(1-\mu)^{2}+\mu^{4}}{\mu(1-\mu)^{2}+\mu(1-\mu)^{2}+\mu^{3}}=1\\
%=\frac{0.1+0.3}{0.1+0.3}
%
%&p(rain|S_{2})=\lim_{\mu\rightarrow 1}\\
%&\frac{\mu(1-\mu)^{2}+(1-\mu)^{3}+\mu^{3}+\mu^{2}(1-\mu)}{\mu(1-\mu)+(1-\mu)^{2}+\mu^{2}+\mu(1-\mu)}=1\\
%
%&p(rain|S_{3})=\lim_{\mu\rightarrow 1}\\
%&\frac{\mu^{2}(1-\mu)+\mu(1-\mu)^{2}+\mu^{2}(1-\mu)+\mu^{2}(1-\mu)}{\mu^{2}+\mu(1-\mu)+\mu(1-\mu)+\mu(1-\mu)}=0\\
%\end{align*}
%%%%%%
\end{example}
%
%%%%%%%%%%%%%%%%%%%%%%%%%%%%%%%%%%%%%%%%%%%%%%%%%%%%%%%%%%%%%%%%%%%%%%%%%%%%%%%%%%%%%%%%%%%%%%%%%%%%%%%%%%%%%%%%%%%%%%%%%%%%%%%%%%%%%%%%%%%%%%%%%%%%%%%%%%%%%%%%%%%%%%%%%%%%%%%%%%%%%%%%%%%%%%%%%%%%%%%%%%%%%%%%%%%%%%%%%%%%%%
\subsection{Counterfactuals}
\emph{Would England have won the match against Argentina at the 1986 World Cup if Diego Maradona had not used his hand to score the first goal?} Reasoning with this kind of false and imaginary conditional statement is often called counterfactual reasoning. Let $\{\lim_{\mu\rightarrow 1}p(\Delta|M,\mu), p(M)\}$ be a logical model such that $\mu\rightarrow 1$. This section demonstrates that the certain inference on the logical model is a natural model of counterfactual reasoning.
\par
Table \ref{tab:data} shows data on four football matches characterised by four attributes: $goal$, $home$, $opponent$, $win\in\{0,1\}$. They are, respectively, facts about whether our teammate Alice scored a goal or not, whether the game was played at home or not, whether the opponent was 0 (meaning Belgium) or 1 (meaning Brazil), and whether our team won or not. Now, we consider the following question.
\begin{quote}
Our team lost the home game without Alice's goal against Belgium, i.e., $m_{1}$. Would we have won if Alice had scored a goal in this match?
\end{quote}
This question does not have a straightforward answer because it is a counterfactual with respect to the data. Indeed, the set of attributes, i.e., $(goal=1,home=1,opponent=0)$, of the counterfactual does not appear in the data.
\par
As long as the counterfactual does not exist in the data, it is reasonable to realise counterfactual reasoning based on the facts most similar to the counterfactual \cite{pearl:18}. The counterfactual shares attributes $(home=1,opponent=0)$ with $m_{1}$, $(goal=1,home=0)$ with $m_{2}$, $(goal=1,opponent=0)$ with $m_{3}$ and $(goal=1)$ with $m_{4}$. The data thus indicates that $m_{1}, m_{2}$ and $m_{3}$ are most similar to the counterfactual in terms of the number of shared attributes. Since the team won in $m_{2}$ and $m_{3}$, it is reasonable to conclude that, given the counterfactual, the probability of winning is 2/3. Here, readers might think that $m_{1}$ should be excluded from the most similar facts because, in the counterfactual, we look at the situation in which Alice scored a goal. However, $m_{1}$ contains important information because it is empirically true that the probability of winning with Alice's goal is positively affected by the fact that we won without Alice's goal and negatively affected by the fact that we lost without Alice's goal.
\par
Interestingly, the idea of counterfactual reasoning is naturally modelled by the logical model. The predictive probability of winning given the counterfactual is calculated as follows.
\begin{align*}
&p(win|goal,home,\lnot opp.)\\
&=\lim_{\mu\rightarrow 1}\frac{\sum_{m}p(goal|m)p(home|m)p(\lnot opp.|m)p(win|m)p(m)}{\sum_{m}p(goal|m)p(home|m)p(\lnot opp.|m)p(m)}\\
&=\lim_{\mu\rightarrow 1}\frac{\mu^{2}(1-\mu)^{2}+\mu^{3}(1-\mu)+\mu^{3}(1-\mu)+\mu(1-\mu)^{3}}{\mu^{2}(1-\mu)+\mu^{2}(1-\mu)+\mu^{2}(1-\mu)+\mu(1-\mu)^{2}}\\
&=\frac{2}{3}
%\\
%&=\lim_{\mu\rightarrow 1}\frac{\mu(1-\mu)+\mu^{2}+\mu^{2}+(1-\mu)^{2}}{\mu+\mu+\mu+(1-\mu)}=\frac{2}{3}
\end{align*}
The denominator of the predictive probability turns out to equal the number of facts most similar to the counterfactual, i.e., $m_{1}$, $m_{2}$ and $m_{3}$, whereas the numerator turns out to equal the number of wins from the three games, i.e., $m_{2}$ and $m_{3}$. Note that only the logical model with $\mu\rightarrow 1$ successfully formalises the idea of counterfactual reasoning.
\begin{table}[t]%[tbph]
\caption{Prior distribution over four football matches.}
\label{tab:data}
\begin{center}
\begin{tabular}{c|c|cccc}
& $p(M)$ & $goal$ & $home$ & $opponent$ & $win$\\\hline
$m_{1}$ & $0.25$ & $0$ & $1$ & $0$ & $0$\\
$m_{2}$ & $0.25$ & $1$ & $1$ & $1$ & $1$\\
$m_{3}$ & $0.25$ & $1$ & $0$ & $0$ & $1$\\
$m_{4}$ & $0.25$ & $1$ & $0$ & $1$ & $0$
\end{tabular}
\end{center}
%\vspace{-1em}
\end{table}
\par
Our approach for counterfactual reasoning essentially differs from Pearl \cite{pearl:18} and Lewis \cite{lewis:73}. Our approach is data-driven, whereas Pearl's approach is model-driven in the sense that it assumes a causal diagram. Our approach is based on probability theory, whereas Lewis's approach is based on the possible-worlds semantics. Although a formal comparison is difficult, Table \ref{tab:comparison} shows that there are some counterparts between the two approaches.
\begin{table}[ht]
\caption{Correspondence with Lewis' counterfactuals.}
\centering
\begin{tabular}{|c|c|}
\hline
Lewis' counterfactuals & Our counterfactuals \\\hline
Possible worlds & Probability distribution $p(M)$\\
Our world(s) & Model(s) $\llbracket\Delta\rrbracket$\\
Most similar world(s) & Approximate model(s) $(\!(\Delta)\!)$\\
Counterfactual $\Delta>\alpha$ & Predictive distribution $p(\alpha|\Delta)$ \\\hline
% MAP $\hat{w}=\argmax_{w}p(w|\Delta)$
\end{tabular}
\label{tab:comparison}
\end{table}
%
%
%%%%%%%%%%%%%%%%%%%%%%%%%%%%%%%%%%%%%%%%%%%%%%%%%%%%%%%%%%%%%%%%%%%%%%%%%%%%%%%%%%%%%%%%%%%%%%%%%%%%%%%%%%%%%%%%%%%%%%%%%%%%%%%%%%%%%%%%%%%%%%%%%%%%%%%%%%%%%%%%%%%%%%%%%%%%%%%%%%%%%%%%%%%%%%%%%%%%%%%%%%%%%%%%%%%%%%%%%%%%%%%%%%%%%%%%%%%%%%%%%%%%%%%%%%%%%%%%%%%%%%%%%%%%%%%%%%%%%%%%%%%%%%%

\section{Conclusions and Discussion}\label{sec:discussion}
In this paper, we introduced a generative model of the logical interpretation that defines the process by which the truth values of formulae are generated probabilistically from data about states of the world. We showed that it is a theory of reasoning that deals with several reasoning problems such as statistical reasoning, logical reasoning, paraconsistent reasoning and counterfactual reasoning.
\par
One of the limitations of the current work is that it is still unclear how our generative model relates to other types of reasoning studied in AI such as nonmonotonic reasoning, abductive reasoning, predictive reasoning and practical reasoning. We will extend the logical model to deal with them in a unified approach.
%%%%%%%%%%%%%%%%%%%%%%%%%%%%%%%%%%%%%%%%%%%%%%%%%%%%%%%%%%%%%%%%%%%%%%%%%%%%%%%%%%%%%%%%%%%%%%%%%%%%%%%%%%%%%%%%%%%%%%%%%%%%%%%%%%%%%%%%%%%%%%%%%%%%%%%%%%%%%%%%%%%%%%%%%%%%%%%%%%%%%%%%%%%%%%%%%%%%%%%%%%%%%%%%%%%%%%%%%%%%%%%%%%%%%%%%%%%%%%%%%%%%%%%%%%%%%%%%%%
%% The file named.bst is a bibliography style file for BibTeX 0.99c
\bibliographystyle{named}
\bibliography{btxkido}

\end{document}